\documentclass{article}
\usepackage{spconf,amsmath,graphicx}

\def\L{{\cal L}}

\usepackage{mathtools}
\usepackage{amsmath,amsopn,amssymb}
\usepackage{amsthm} 
\usepackage{graphicx,xspace,color,soul}
\usepackage{epsfig,subfigure}
\usepackage{longtable,multirow}
\usepackage{array,float}
\usepackage{algorithmic}
\usepackage[linesnumbered, ruled]{algorithm2e}
\usepackage{bm,epstopdf}
\usepackage{tabularx}

\newtheorem{theorem}{Theorem}
\newtheorem{definition}{Definition}
\newtheorem{lemma}{Lemma}

\newtheorem{corollary}{Corollary}
\renewcommand{\vec}[1]{\mathbf{#1}}

\newcolumntype{Y}{>{\centering\arraybackslash}X}

\def\eg{\textit{e.g.}}
\def\ie{\textit{i.e.}}
\def\0{{\mathbf 0}}
\def\1{{\mathbf 1}}

\def\f{{\mathbf f}}

\def\v{{\mathbf v}}
\def\z{{\mathbf z}}

\def\B{{\mathbf B}}

\def\D{\mathbf{D}}

\def\I{{\mathbf I}}
\def\L{{\mathbf L}}
\def\M{{\mathbf M}}
\def\Q{{\mathbf Q}}
\def\S{{\mathbf S}}

\def\W{{\mathbf W}}

\def\balpha{{\boldsymbol \alpha}}

\def\cS{{\mathcal S}}
\def\cG{{\mathcal G}}

\title{Graph Metric Learning via Gershgorin Disc Alignment}
\name{Cheng Yang$^\dagger$, Gene Cheung$^\dagger$, Wei Hu$^\ddagger$
\thanks{Gene Cheung acknowledges the support of the NSERC grants [RGPIN-2019-06271],  [RGPAS-2019-00110].}
\thanks{Wei Hu acknowledges the support of National Natural Science Foundation of China (61972009) and Beijing Natural Science Foundation (4194080).}
}
\address{$^\dagger$Department of Electrical Engineering \& Computer Science, York University, Toronto, Canada\\
         $^\ddagger$Wangxuan Institute of Computer Technology, Peking University, Beijing, China}

\begin{document}
\ninept
\maketitle
\begin{abstract}
We propose a 
general projection-free metric learning framework, where the minimization objective $\min_{\M \in \cS} Q(\M)$ is a convex differentiable function of the metric matrix $\M$, and $\M$ resides in the set $\cS$ of generalized graph Laplacian matrices for connected graphs with positive edge weights and node degrees.
Unlike low-rank metric matrices common in the literature, $\cS$ includes the important positive-diagonal-only matrices as a special case in the limit. 
The key idea for fast optimization is to rewrite the positive definite cone constraint in $\cS$ as signal-adaptive linear constraints via Gershgorin disc alignment, so that the alternating optimization of the diagonal and off-diagonal terms in $\M$ can be solved efficiently as linear programs via Frank-Wolfe iterations.
We prove that left-ends of the Gershgorin discs can be aligned perfectly using the first eigenvector $\v$ of $\M$, which we update iteratively using Locally Optimal Block Preconditioned Conjugate Gradient (LOBPCG) with warm start as diagonal / off-diagonal terms are optimized.
Experiments show that our efficiently computed graph metric matrices outperform metrics learned using competing methods in terms of classification tasks. 
\end{abstract}
\section{Introduction}
\label{sec:intro}
Given a feature vector $\f_i \in \mathbb{R}^K$ per sample $i$, a \textit{metric} matrix $\M \in \mathbb{R}^{K \times K}$ defines the \textit{feature distance} (\textit{Mahalanobis distance} \cite{mahalanobis1936}) between two samples $i$ and $j$ in a feature space as
$(\mathbf{f}_i - \mathbf{f}_j)^{\top} \mathbf{M} (\mathbf{f}_i - \mathbf{f}_j)$, where $\M$ is commonly assumed to be positive definite (PD). 
\textit{Metric learning}---identifying the best metric $\M$ minimizing a chosen objective function $Q(\M)$ subject to $\M \succ 0$---has been the focus of many recent machine learning research efforts \cite{weinberger09LMNN,qi09icml,lim13icml,liu15aaai,zadeh16GMML}.

One key challenge in metric learning is to satisfy the positive (semi-)definite (PSD) cone constraint $\M \succ 0$ ($\M \succeq 0$) when minimizing $Q(\M)$ in a computation-efficient manner.
A standard approach is alternating gradient-descent / projection (\eg, \textit{proximal gradient} (PG) \cite{Parikh31}), where a descent step $\balpha$ from current solution $\M^t$ at iteration $t$ in the direction of the negative gradient $-\nabla Q(\M^t)$ is followed by a projection $\mathrm{Pr}()$ back to the PSD cone, \ie, $\M^{t+1} := \mathrm{Pr} \left( \M^t - \balpha \nabla Q(\M^t) \right)$.
However, projection $\mathrm{Pr}()$ typically requires eigen-decomposition of $\M$ and soft-thresholding of its eigenvalues, which is computation-expensive.

Recent methods consider alternative search spaces of matrices such as sparse or low-rank matrices to ease optimization \cite{qi09icml,lim13icml,liu15aaai,mu16aaai,zhang17aaai}. 
While efficient, the assumed restricted search spaces often degrade the quality of sought metric $\M$ in defining the Mahalanobis distance. 
For example, low-rank methods explicitly assume reducibility of the $K$ available features to a lower dimension, and hence exclude the simple yet important weighted feature metric case where $\M$ contains only positive diagonal entries \cite{yang2018apsipa}, \ie, $(\mathbf{f}_i - \mathbf{f}_j)^{\top} \mathbf{M} (\mathbf{f}_i - \mathbf{f}_j) = \sum_{k} m_{k,k} (f_i^k - f_j^k)^2$, $m_{k,k} > 0, \forall k$. 
We show in our experiments that computed metrics by these methods may result in inferior performance for selected applications. 

In this paper, we propose a 
metric learning framework that is both general and projection-free, capable of optimizing any convex differentiable objective $Q(\M)$. 
Compared to low-rank methods, our framework is more encompassing and includes positive-diagonal metric matrices as a special case in the limit\footnote{As the inter-feature correlations tend to zero, only graph self-loops expressing relative importance among the $K$ features remain, and the generalized graph Laplacian matrix tends to diagonal.}. 
The main idea is as follows.
First, we define a search space $\cS$ of \textit{general graph Laplacian} matrices \cite{biyikoglu2005nodal}, each corresponding to a connected graph with positive edge weights and node degrees.
The underlying graph edge weights capture pairwise correlations among the $K$ features, and the self-loops designate relative importance among the features. 

Assuming $\M \in \cS$, we next rewrite the PD cone constraint as signal-adaptive linear constraints via \textit{Gershgorin disc alignment} \cite{bai19icassp,bai19tsp}:  
i) compute scalars $s_k$'s from previous solution $\M^t$ that align the Gershgorin disc left-ends of matrix $\S \M^t \S^{-1}$, where $\S = \mathrm{diag}(s_1, \ldots, s_K)$, ii) derive scaled linear constraints using $s_k$'s to ensure PDness of the next computed metric $\M^{t+1}$ via the Gershgorin Circle Theorem (GCT) \cite{gahc}.
Linear constraints mean that our proposed alternating optimization of the diagonal and off-diagonal terms in $\M$ can be solved speedily as linear programs (LP) \cite{co1998} via Frank-Wolfe iterations \cite{pmlr-v28-jaggi13}.
We prove that for any metric $\M^t$ in $\cS$, using scalars $s_k = 1 / v_k$ can perfectly align Gershgorin disc left-ends for matrix $\S \M^t \S^{-1}$ at the smallest eigenvalue $\lambda_{\min}$, where $\M^t \v = \lambda_{\min} \v$.
We efficiently update $\v$ iteratively using \textit{Locally Optimal Block Preconditioned Conjugate Gradient} (LOBPCG) \cite{Knyazev01} with warm start as diagonal / off-diagonal terms are optimized.
Experiments show that our computed graph metrics outperform metrics learned using competing methods in terms of classification tasks.



\section{Review of Spectral Graph Theory}
\label{sec:graph}


We consider an undirected graph $ \mathcal{G}=\{\mathcal{V},\mathcal{E}, \mathbf{W}\} $ composed of a node set $ \mathcal{V} $ of cardinality $|\mathcal{V}|=N$, an edge set $ \mathcal{E} $ connecting nodes, and a weighted adjacency matrix $\mathbf{W}$. Each edge $(i,j) \in \mathcal{E}$ has a positive weight $w_{i,j} > 0$ which reflects the degree of similarity (correlation) between nodes $i$ and $j$.
Specifically, it is common to compute edge weight $w_{i,j}$ as the exponential of the negative feature distance $\delta_{i,j}$ between nodes $i$ and $j$ \cite{shuman13spm}: 
\begin{align}
w_{i,j} = \exp \left( - \delta_{i,j} \right) 
\label{eq:edgeWeight}
\end{align}
Using \eqref{eq:edgeWeight} means $w_{i,j} \in (0, 1]$ for $\delta_{i,j} \in [0,\infty)$.
We discuss feature distance $\delta_{i,j}$ in the next section.

There may be \textit{self-loops} in graph $\cG$, \ie, $\exists i$ where $w_{i,i} > 0$, and the corresponding diagonal entries of $\W$ are positive.
The \textit{combinatorial graph Laplacian} \cite{shuman13spm} is defined as $ \L:=\D-\W $, where $ \D $ is the \textit{degree matrix}---a diagonal matrix where $ d_{i,i} = \sum_{j=1}^N w_{i,j}$. 
A \textit{generalized graph Laplacian} \cite{biyikoglu2005nodal} accounts for self-loops in $\cG$ also and is defined as $\L_g = \D - \W + \mathrm{diag}(\W) $, where $\mathrm{diag}(\W)$ extracts the diagonal entries of $\W$.
Alternatively we can write $\L_g = \D_g - \W$, where the \textit{generalized degree matrix} $\D_g = \D + \mathrm{diag}(\W)$ is diagonal.


\section{Graph Metric Learning}
\label{sec:learning}
\subsection{Graph Metric Matrices}


We first define the search space of metric matrices for our optimization framework. 
We assume that associated with each sample $i$ is a length-$K$ feature vector $\f_i \in \mathbb{R}^K$. 
A metric matrix $\M \in \mathbb{R}^{K \times K}$ defines the feature distance $\delta_{i,j}(\M)$---the \textit{Mahalanobis distance} \cite{mahalanobis1936}---between samples $i$ and $j$ as: 
\begin{equation}
    \delta_{i,j}(\M) = (\f_i-\f_j)^{\top} \mathbf{M} (\f_i-\f_j)
    \label{eq:featureDist}
\end{equation}
We require $\M$ to be a \textit{positive definite} (PD) matrix\footnote{By definition of a metric \cite{fsp2014}, $(\f_i-\f_j)^{\top} \M (\f_i-\f_j) > 0$ if $\f_i - \f_j \neq \0$.}. 
%
%
The special case where $\M$ is diagonal with strictly positive entries
was studied in \cite{yang2018apsipa}. 
Instead, we study here a more general case: $\M$ must be a \textit{graph metric matrix}, which we define formally as follows.
\begin{definition}
A PD symmetric matrix $\M$ is a graph metric if it is a generalized graph Laplacian matrix with positive edge weights and node degrees for an irreducible graph.
\end{definition}

\noindent
\textbf{Remark}: A generalized graph Laplacian matrix $\M$ with positive degrees means $m_{i,i} > 0$; in graph terminology, each graph node $i$ may have a self-loop, but its loop weight $w_{i,i}$ must satisfy $w_{i,i} > - \sum_{j\,|\, j \neq i} w_{i,j}$.
Positive edge weights means $m_{i,j} \leq 0, i \neq j$.
Irreducible graph \cite{irregraph}
essentially means that any graph node can \textit{commute} with any other node. 


\subsection{Problem Formulation}

Denote by $\cS$ the set of all graph metric matrices.
We pose an optimization problem for $\M$: 
find the optimal graph metric $\M$ in $\cS$---leading to inter-sample distances $\delta_{i,j}(\M)$ in \eqref{eq:featureDist}---that yields the smallest value of a convex differential objective $Q(\{\delta_{i,j}(\M)\})$:

\begin{align}
\min_{\mathbf{M} \in \cS}
    Q\left(\{\delta_{i,j}(\M)\} \right),
    ~~~\mbox{s.t.}~~ \mathrm{tr}(\M) \leq C
\label{eq:optimize_M}
\end{align}
where $C > 0$ is a chosen parameter. 
Constraint $\mathrm{tr}(\M) \leq C$ is needed to avoid pathological solutions with infinite feature distances, \ie, $\delta_{i,i}(\M) = \infty$. 
For stability, we assume also that the objective is lower-bounded, \ie, $\min_{\M \in \cS} Q(\{\delta_{i,j}(\M)\}) \geq \kappa > -\infty$ for some constant $\kappa$.


Our strategy to solve \eqref{eq:optimize_M} is to optimize $\mathbf{M}$'s diagonal and off-diagonal terms alternately using Frank-Wolfe iterations \cite{pmlr-v28-jaggi13}, where each iteration is solved as an LP until convergence. 
We discuss first the initialization of $\M$, then the two optimizations in order.
For notation convenience, we will write the objective simply as $Q(\M)$, with the understanding that metric $\M$ affects first the feature distances $\delta_{i,j}(\M)$, which in turn determine the objective $Q(\{\delta_{i,j}(\M)\})$.

\subsection{Initialization of $\M$}

We first initialize a valid graph metric $\M^0$ as follows:
\begin{enumerate}
\item Initialize each diagonal term $m_{i,i}^0 := C/K$.
\item Initialize off-diagonal terms $m_{i,j}^0$, $i \neq j$, as:
\begin{align}
m_{i,j}^0 := \left\{ \begin{array}{ll} 
-\epsilon & \mbox{if}~ j=i \pm 1 \\ 
 0 & \mbox{o.w.}
 \end{array} \right.
 \end{align}
\end{enumerate}
where $\epsilon > 0$ is a small parameter. 
Initialization of the diagonal terms ensures that constraints $\mathrm{tr}(\M^0) \leq C$, $\M^0 \succ 0$ and $m^0_{i,i} > 0$ are satisfied.
Initialization of the off-diagonal terms ensures that $\M^0$ is symmetric and irreducible, and constraint $m_{i,j}^0 \leq 0$, $i\neq j$, is satisfied; \ie, $\M^0$ is a generalized graph Laplacian matrix for graph with positive edge weights.
We can hence conclude that initial $\M^0$ is a graph metric, \ie, $\M^0 \in \cS$.

\subsection{Optimization of Diagonal Terms} 
\label{subsection:opt_diag}

When optimizing $\M$'s diagonal terms $m_{i,i}$,  \eqref{eq:optimize_M} becomes
\begin{align}
&\min_{\{m_{i,i}\}} ~~
    Q(\M)
    \label{eq:optimize_diagonal} \\
& \text{s.t.} \quad \,\M \succ 0; \;\;\;
\sum_{i} m_{i,i} \leq C; ~~~
m_{i,i} > 0, \,\forall i \nonumber
\end{align}
where $\mathrm{tr}(\M) = \sum_i m_{i,i}$. 
Because the diagonal terms do not affect the irreducibility of matrix $\M$, the only requirements for $\M$ to be a graph metric are: i) $\M$ must be PD, and ii) diagonals must be strictly positive.

\subsubsection{Gershgorin-based Reformulation}

To efficiently enforce the PD constraint $\M \succ 0$, we derive sufficient (but not necessary) linear constraints using the \textit{Gershgorin Circle Theorem} (GCT) \cite{gahc}.
By GCT, each eigenvalue $\lambda$ of a real matrix $\M$ resides in at least one \textit{Gershgorin disc} $\Psi_i$, corresponding to row $i$ of $\M$, with center $c_i = m_{i,i}$ and radius $r_i = \sum_{j \,|\, j\neq i} |m_{i,j}|$, \ie,
\begin{align}
\exists i ~~\mbox{s.t.}~~
c_i - r_i \leq \lambda \leq c_i + r_i
\end{align}
Thus a sufficient condition to ensure $\M$ is PD (smallest eigenvalue $\lambda_{\min} > 0$) is to ensure that all discs' left-ends are strictly positive, \ie,
\begin{align}
0 < \min_i c_i - r_i \leq \lambda_{\min}
\end{align}
This translates to a linear constraint for each row $i$:
\begin{align}
m_{i,i} \geq \sum_{j \,|\, j \neq i} |m_{i,j}| + \rho,
~~~~~~ \forall i \in \{1, \ldots, K\}
\label{eq:GCT_linConst}
\end{align}
where $\rho > 0$ is a sufficiently small parameter. 

However, GCT lower bound $\min_i c_i - r_i$ for $\lambda_{\min}$ is often loose.
When optimizing $\M$'s diagonal terms, enforcing \eqref{eq:GCT_linConst} directly means that we are searching for $\{m_{i,i}\}$ in a much smaller space than the original space $\{\M ~|~ \M \succ 0\}$ in \eqref{eq:optimize_diagonal}, resulting in an inferior solution. 
As an illustration, consider the following example matrix $\M$:
\begin{align}
\M = \left[ \begin{array}{ccc}
2 & -2 & -1 \\
-2 & 5 & -2 \\
-1 & -2 & 4
\end{array} \right]
\label{eq:exM}
\end{align}
Gershgorin disc left-ends $m_{i,i} - \sum_{j \,|\, j\neq i} |m_{i,j}|$ for this matrix are $\{-1, 1, 1\}$, of which $-1$ is the smallest.
Thus the diagonal terms $\{2, 5, 4\}$ do not meet constraints \eqref{eq:GCT_linConst}.
However, $\M$ is PD, since its smallest eigenvalue is $\lambda_{\min} = 0.1078 > 0$.

\subsubsection{Gershgorin Disc Alignment}

To derive more appropriate linear constraints---thus more suitable search space when solving $\min_{\M \in \cS} \Q(\M)$, we examine instead the Gershgorin discs of a similar-transformed matrix $\B$ from $\M$, \ie,
\begin{align}
\B = \S \M \S^{-1}
\label{eq:similarTrans}
\end{align}
where $\S = \mathrm{diag}(s_1, \ldots, s_K)$ is a diagonal matrix with scalars $s_1, \ldots, s_K$ along its diagonal, $s_k > 0, \,\forall k$.
$\B$ has the same eigenvalues as $\M$, and thus the smallest Gershgorin disc left-end, $\min_i b_{i,i} - \sum_{j \,|\, j\neq i} |b_{i,j}|$, for $\B$ is also a lower bound for $\M$'s smallest eigenvalue $\lambda_{\min}$. 
\textit{
Our goal is to derive tight $\lambda_{\min}$ lower bounds by adapting to good solutions to \eqref{eq:optimize_diagonal}---by appropriately choosing scalars $s_1, \ldots, s_K$ used to define similar-transformed $\B$ in \eqref{eq:similarTrans}.
}

Specifically, given scalars $s_1, \ldots, s_K$, a disc $\Psi_i$ for $\B$ has center $m_{i,i}$  and   radius $s_i \sum_{j \,|\, j \neq i} |m_{i,j}|/s_j$. 
Thus to ensure $\B$ is PD (and hence $\M$ is PD), we can write similar linear constraints as \eqref{eq:GCT_linConst}:
\begin{align}
m_{i,i} \geq s_i \sum_{j \,|\, j \neq i} \frac{|m_{i,j}|}{s_j} + \rho, 
~~~~ \forall i \in \{1, \ldots, K\}
\label{eq:scaled_linConst}
\end{align}
It turns out that given a graph metric $\M$, there exist scalars $s_1, \ldots, s_K$ such that all Gershgorin disc left-ends are aligned at the same value $\lambda_{\min}$. 
We state this formally as a theorem.

\begin{theorem}
Let $\M$ be a graph metric matrix. 
There exist strictly positive scalars $s_1, \ldots, s_K$ such that all Gershgorin disc left-ends of $\B = \S \M \S^{-1}$ are aligned exactly at the smallest eigenvalue, \ie, $b_{i,i} - \sum_{j \,|\, j\neq i} |b_{i,j}| = \lambda_{\min}, \forall i$.
\end{theorem}
In other words, for matrix $\B$ the Gershgorin lower bound $\min_i c_i - r_i$ is \textit{exactly} $\lambda_{\min}$, and the bound is the tightest possible.
The important corollary is the following:
\begin{corollary}
For any graph metric matrix $\M$, which by definition is PD, there exist scalars $s_1, \ldots, s_K$ where $\M$ is feasible using linear constraints in \eqref{eq:scaled_linConst}.
\end{corollary}
\begin{proof}
By Theorem 1, let $s_1, \ldots, s_K$ be scalars such that all Gershgorin disc left-ends of $\B = \S \M \S^{-1}$ align at $\lambda_{\min}$. 
Thus
\vspace{-0.1in}
\begin{align}
\forall i, ~~
m_{i,i} - s_i \sum_{j\,|\,j\neq i} \frac{|m_{i,j}|}{s_j} = \lambda_{\min} > 0
\end{align}
where $\lambda_{\min} > 0$ since $\M$ is PD.
Hence $\M$ must also satisfy \eqref{eq:scaled_linConst} for all $i$ for sufficiently small $\rho > 0$.
\end{proof}
Continuing our earlier example, using $s_1 = 0.7511$, $ s_2 = 0.4886$ and $s_3=0.4440$, we see that  $\B = \S \M \S^{-1}$ for $\M$ in \eqref{eq:exM} has all disc left-ends aligned at $\lambda_{\min} = 0.1078$.
Hence using these scalars and constraints \eqref{eq:scaled_linConst}, diagonal terms $\{2, 5, 4\}$ now constitute a feasible solution.

To prove Theorem 1, we first establish the following lemma. 


\begin{lemma}
There exists a first eigenvector $\v$ with strictly positive entries for a graph metric matrix $\M$.
\end{lemma}

\begin{proof}
By definition, graph metric matrix $\M$ is a generalized graph Laplacian $\L_g = \D_g - \W$ with positive edge weights in $\W$ and positive degrees in $\D_g$.
Let $\v$ be the first eigenvector of $\M$, \ie, 
\begin{align}
\M \v &= \lambda_{\min} \v \nonumber \\
(\D_g - \W) \v &= (\lambda_{\min} \I) \v \nonumber \\
\D_g \v &= (\W + \lambda_{\min} \I) \v \nonumber \\
\v &= \D_g^{-1} (\W + \lambda_{\min} \I) \v \nonumber
\end{align}
where $\lambda_{\min} > 0$ since $\M$ is PD.
Since the matrix on the right contains only non-negative entries and $\W$ is an irreducible matrix, $\v$ is a positive eigenvector by the Perron-Frobenius Theorem \cite{ma2012}.
\end{proof}

We now prove Theorem 1 as follows.
\begin{proof}
Denote by $\v$ a strictly positive eigenvector corresponding to graph metric matrix $\M$'s smallest eigenvalue $\lambda_{\min}$. 
Define $\S = \mathrm{diag}(1/v_1, \ldots, 1/v_K)$.
Then,
\begin{align}
\S \M \S^{-1} \S \v = \lambda_{\min} \S \v 
\end{align}
where $\S \v = \1 = [1, \ldots, 1]^{\top}$.
Let $\B = \S \M \S^{-1}$.
Then,
\begin{align}
\B \1 = \lambda_{\min} \1 
\label{eq:GCT_proof}
\end{align}
\eqref{eq:GCT_proof} means that
\vspace{-0.05in}
\begin{align}
b_{i,i} + \sum_{j \,|\, j \neq i} b_{i,j} &= \lambda_{\min}, ~~~ \forall i \nonumber 
\end{align}
Note that the off-diagonal terms $b_{i,j} = (v_i/v_j) m_{i,j} \leq 0$, since i) $\v$ is strictly positive and ii) off-diagonal terms of graph metric $\M$ satisfy $m_{i,j} \leq 0$. 
Thus,
\begin{align}
b_{i,i} - \sum_{j \,|\, j \neq i} |b_{i,j}| &= \lambda_{\min}, ~~~ \forall i
\end{align}
Thus defining $\S = \mathrm{diag}(1/v_1, \ldots, 1/v_K)$ means $\B = \S \M \S^{-1}$ has all its Gershgorin disc left-ends aligned at $\lambda_{\min}$. 
\end{proof}


Thus, using a positive first eigenvector $\v$ of a graph metric $\M$, one can compute corresponding scalars $s_k = 1/v_k$ to align all disc left-ends of $\B = \S \M \S^{-1}$ at $\lambda_{\min}$, and $\M$ satisfies \eqref{eq:scaled_linConst} by Corollary 1.
Note that these scalars are \textit{signal-adaptive}, \ie, $s_k$'s depend on $\v$, which is computed from $\M$.
Our strategy then is to derive scalars $s_k^t$'s from a good solution $\M^{t-1}$, optimize for a better solution $\M^t$ using scaled Gershgorin linear constraints \eqref{eq:scaled_linConst}, derive new scalars again until convergence.
Specifically,
\begin{enumerate}
\item Given scalars $s_k^t$'s, identify a good solution $\M^t$ minimizing objective $Q(\M)$ subject to \eqref{eq:scaled_linConst}, \ie, 
\begin{align}
\min_{\{m_{i,i}\}} &
    Q \left( \M \right) \label{eq:optimize_diagonal2}  \\
\text{s.t.} & ~~ m_{i,i} \geq s_i \sum_{j \,|\, j \neq i} \frac{|m_{i,j}|}{s_j} + \rho, \forall i;
~~~\sum_{i} m_{i,i} \leq C 
\nonumber
\end{align}
\item Given $\M^t$, update scalars $s_k^{t+1} = 1/v_k^t$ where $\v^t$ is the first eigenvector of $\M^t$.
\item Increment $t$ and repeat until convergence.
\end{enumerate}

When the scalars in \eqref{eq:optimize_diagonal2} are updated as $s_k^{t+1} = 1/v_k^t$ for iteration $t+1$, we show that previous solution $\M^t$ at iteration $t$ remains feasible at iteration $t+1$:

\begin{lemma}
Solution $\M^t$ to \eqref{eq:optimize_diagonal2} in iteration $t$ remains feasible in iteration $t+1$, when scalars $s_i^{t+1}$ for the linear constraints in \eqref{eq:optimize_diagonal2} are updated as $s_i^{t+1} = 1/v_i^t, \forall i$, where $\v^t$ is the first eigenvector of $\M^t$. 
\end{lemma}

\begin{proof}
Using the first eigenvector $\v^t$ of graph metric $\M^t$ at iteration $t$, by the proof of Theorem 1 we know that the Gershgorin disc left-ends of $\B = \S \M^t \S^{-1}$ are aligned at $\lambda_{\min}$.
Since $\M^t$ is a feasible solution in \eqref{eq:optimize_diagonal2}, $\M^t \succ 0$ and $\lambda_{\min} > 0$. 
Thus $\M^t$ is also a feasible solution when scalars are updated as $s_i = 1/v_i^t, \forall i$. 
\end{proof}




The remaining issue is how to best compute first eigenvector $\v^t$ given solution $\M^t$ repeatedly.
For this task, we employ \textit{Locally Optimal Block Preconditioned Conjugate Gradient} (LOBPCG) \cite{Knyazev01}, a fast iterative algorithm known to compute extreme eigenpairs efficiently. 
Further, using previously computed eigenvector $\v^{t-1}$ as an initial guess, LOBPCG benefits from warm start when computing $\v^t$, reducing its complexity in subsequent iterations \cite{Knyazev01}.

\subsubsection{Frank-Wolfe Algorithm}

To solve \eqref{eq:optimize_diagonal2}, we employ the Frank-Wolfe algorithm \cite{pmlr-v28-jaggi13} that iteratively linearizes the objective $Q(\M)$ using its gradient $\nabla Q(\M^t)$ with respect to diagonal terms $\{m_{i,i}\}$, computed using previous solution $\M^t$, \ie,
\begin{align}
\nabla Q(\M^t) = \left. \left[ \begin{array}{c}
\frac{\partial Q(\M)}{\partial m_{1,1}} \\
\vdots \\
\frac{\partial Q(\M)}{\partial m_{K,K}}
\end{array} \right] \right|_{\M^t}
\label{eq:gradM}
\end{align}

Given gradient $\nabla Q(\M^t)$, optimization \eqref{eq:optimize_diagonal2} becomes a \textit{linear program} (LP) at each iteration $t$:
\begin{align}
\min_{\{m_{i,i}\}} &
    \mathrm{vec}(\{m_{i,i}\})^\top ~\nabla Q(\M^t) \label{eq:fwlp} \\
\text{s.t.} & ~~ m_{i,i} \geq s_i \sum_{j \,|\, j \neq i} \frac{|m_{i,j}^t|}{s_j} + \rho, ~~\forall i;
~~~\sum_{i} m_{i,i} \leq C.
\nonumber
\end{align}
where $\mathrm{vec}(\{m_{i,i}\}) = [m_{1,1} ~ m_{2,2} ~\ldots ~ m_{K,K}]^{\top}$ is a vector composed of diagonal terms $\{m_{i,i}\}$, and $m_{i,j}^t$ are off-diagonal terms of previous solution $\M^t$.
LP \eqref{eq:fwlp} can be solved efficiently using known fast algorithms such as Simplex \cite{co1998} and interior point method \cite{co2009}. 
When a new solution $\{m_{i,i}^{t+1}\}$ is obtained, gradient $\nabla Q(\M^{t+1})$ is updated, and LP \eqref{eq:fwlp} is solved again until convergence.

\subsection{Optimization of Off-diagonal Entries}

For off-diagonal entries of $\mathbf{M}$, we design a block coordinate descent algorithm, which optimizes one row / column at a time. 

\subsubsection{Block Coordinate Iteration}
 
First, we divide $\mathbf{M}$ into four sub-matrices:
\begin{equation}
\mathbf{M} = \begin{bmatrix}
m_{1,1} & \mathbf{M}_{1,2} \\
\mathbf{M}_{2,1} & \mathbf{M}_{2,2} 
\end{bmatrix},
\label{eq:submatrix}
\end{equation}
where $m_{1,1} \in \mathbb{R}$, $\mathbf{M}_{1,2} \in \mathbb{R}^{1 \times (K-1)}$, $\mathbf{M}_{2,1} \in \mathbb{R}^{(K-1) \times 1}$ and $\mathbf{M}_{2,2} \in \mathbb{R}^{(K-1) \times (K-1)}$. 
Assuming $\mathbf{M}$ is symmetric,  $\mathbf{M}_{1,2} = \mathbf{M}_{2,1}^{\top}$.     
We optimize $\mathbf{M}_{2,1}$ in one iteration, \ie,
\begin{align}
\min_{\M_{2,1}} ~ Q(\M), ~~~
\mbox{s.t.} ~~ \M \in \cS
\label{eq:optimize_off}
\end{align}
In the next iteration, a different row / column $i$ is selected, and with appropriate row / column permutation, we still optimize the first column off-diagonal terms $\M_{2,1}$ as in \eqref{eq:optimize_off}.

Note that the constraint $\mathrm{tr}(\M) \leq C$ in \eqref{eq:optimize_M} can be ignored, since it does not involve optimization variable $\M_{2,1}$.
For $\M$ to remain in the set $\cS$ of graph metric matrices, i) $\M$ must be PD, ii) $\M$ must be irreducible, and iii) $\M_{2,1} \leq \0$. 

As done for the diagonal terms optimization, we replace the PD constraint with Gershgorin-based linear constraints. 
To ensure irreducibility (\ie, the graph remains connected), we ensure that \textit{at least one} off-diagonal term (say index $\zeta$) in column 1 has magnitude at least $\epsilon > 0$. 
The optimization thus becomes:
\begin{align}
\min_{\M_{2,1}} ~& Q(\M) 
\label{eq:optimize_off2} \\
\mbox{s.t.} ~~& m_{i,i} \geq s_i \sum_{j\,|\,j\neq i} \frac{|m_{i,j}|}{s_j} + \rho, ~~\forall i \nonumber \\
& m_{\zeta,1} \leq -\epsilon; ~~~
\M_{2,1} \leq \0 \nonumber
\end{align}
Essentially any selection of $\zeta$ in \eqref{eq:optimize_off2} can ensure $\M$ is irreducible. 
To encourage solution convergence, we select $\zeta$ as the index of the previously optimized $\M_{2,1}^t$ with the largest magnitude. 

\eqref{eq:optimize_off2} also has a convex differentiable objective with a set of linear constraints. 
We thus employ the Frank-Wolfe algorithm again to iteratively linearize the objective using gradient $\nabla Q(\M^t)$ with respect to off-diagonal $\M_{2,1}$, where the solution in each iteration is solved as an LP. 
We omit the details for brevity. 




\section{Experiments}
\label{sec:results}

We evaluate our proposed metric learning method in classification performance.
Specifically, the objective function $Q(\M)$ we consider here is the \textit{graph Laplacian Regularizer} (GLR) \cite{shuman13spm,pang2017graph}:
\begin{align}
&	Q(\M) = \vec{z}^{\top} \mathbf{L}(\M) \vec{z} =\sum_{i=1}^{N} \sum_{j=1}^{N} w_{i,j}(z_i - z_j)^2 
	\nonumber \\
	= & \sum_{i=1}^{N} \sum_{j=1}^{N} \exp \left\{ -(\f_i-\f_j)^{\top} \M (\f_i - \f_j) \right\} (z_i - z_j)^2.
	\label{eq:GLR}
\end{align}
A small GLR means that signal $\z$ at connected node pairs $(z_i, z_j)$ are similar for a large edge weight $ w_{i,j} $, \ie, $z$ is \textit{smooth} with respect to the variation operator $\L(\M)$. 
GLR has been used in the GSP literature to solve a range of inverse problems, including image denoising \cite{pang2017graph}, deblurring \cite{bai18}, dequantization amd contrast enhancement \cite{liu19contrast}, and soft decoding of JPEG \cite{liu2016random}.


We evaluate our method with the following competing schemes: three metric learning methods that only learn the diagonals of $\mathbf{M}$, \textit{i.e.}, \cite{Zhu:2003:SLU:3041838.3041953}, \cite{maoapsipa16}, and \cite{yang2018apsipa}, and two methods that learn the full matrix $\mathbf{M}$, \textit{i.e.}, \cite{zadeh16GMML} and \cite{DBLP:journals/corr/abs-1907-09138}. 
We perform classification tasks using one of the following two classifiers: 1) a k-nearest-neighbour classifier, and 2) a graph-based classifier with quadratic formulation 
$\min_{\z} \z^\top \L(\M) \z \nonumber ~ \mbox{s.t.} ~ z_i = \hat{z}_i, i \in \mathcal{F}, \mathcal{F} \subset \left \{ 1, \ldots, J \right\}$, 
where $\hat{z}_i$ in subset $\mathcal{F}$ are the observed labels. 
We evaluate all classifiers on \texttt{wine} (3 classes, 13 features and 178 samples), \texttt{iris} (3 classes, 4 features and 150 samples), \texttt{seeds} (3 classes, 7 features and 210 samples), and \texttt{pb} (2 classes, 10 features and 300 samples). 
All experiments were performed in Matlab R2017a on an i5-7500, 8GB of RAM, Windows 10 PC. 
We perform 2-fold cross validation 50 times using 50 random seeds (0 to 49) with one-against-all classification strategy. 
As shown in Table \ref{table:classification}, our proposed metric learning method has the lowest classification error rates with a graph-based classifier.



\begin{table}[htb]
	\centering
	\caption{Classification error rates. (GB=Graph-based classifier.)}
\label{table:classification}
\begin{scriptsize}
\begin{tabular}{|c|c|c|c|c|c|c|c|c|} \hline
\multirow{2}{*}{methods} & \multicolumn{2}{c|}{\texttt{iris}} & \multicolumn{2}{c|}{\texttt{wine}} & \multicolumn{2}{c|}{\texttt{seeds}} & \multicolumn{2}{c|}{\texttt{pb}}\\ \cline{2-9}
& {kNN} & {GB}  & {kNN} & {GB} & {kNN} & {GB} & {kNN} & {GB}\\ \hline
\cite{Zhu:2003:SLU:3041838.3041953} & 4.61  & 4.41 & 3.84 & 4.88  & 7.30 & 7.20 & - & - \\
\cite{maoapsipa16}  & 4.97 & 4.57 & 4.61 & 5.18 & 7.15& 6.93 & 4.46 & 5.04 \\
\cite{yang2018apsipa}  & 5.45 & 5.49 & 4.35  & 4.96 & 7.78 & 7.40 & 5.33 & 4.51\\ \hline
\cite{zadeh16GMML} & 6.12 & 10.40 & \textbf{3.58} & 4.37 & \textbf{6.92} &  6.63 & 4.55 & 4.96\\
\cite{DBLP:journals/corr/abs-1907-09138}    & \textbf{4.35} & 4.80 & 4.12 &  4.36 & 7.77& 7.47 & \textbf{4.44} & 4.24 \\
\textbf{Prop.}  & \textbf{4.35} & \textbf{4.12} & 4.27 & \textbf{4.19} & 7.10 & \textbf{6.61} & 4.8 & \textbf{4.23}\\ \hline
\end{tabular}
\end{scriptsize}
\end{table}

\vfill\pagebreak



\bibliographystyle{IEEEtran}
\bibliography{ref}

\begin{thebibliography}{10}
\providecommand{\url}[1]{#1}
\csname url@samestyle\endcsname
\providecommand{\newblock}{\relax}
\providecommand{\bibinfo}[2]{#2}
\providecommand{\BIBentrySTDinterwordspacing}{\spaceskip=0pt\relax}
\providecommand{\BIBentryALTinterwordstretchfactor}{4}
\providecommand{\BIBentryALTinterwordspacing}{\spaceskip=\fontdimen2\font plus
\BIBentryALTinterwordstretchfactor\fontdimen3\font minus
  \fontdimen4\font\relax}
\providecommand{\BIBforeignlanguage}[2]{{%
\expandafter\ifx\csname l@#1\endcsname\relax
\typeout{** WARNING: IEEEtran.bst: No hyphenation pattern has been}%
\typeout{** loaded for the language `#1'. Using the pattern for}%
\typeout{** the default language instead.}%
\else
\language=\csname l@#1\endcsname
\fi
#2}}
\providecommand{\BIBdecl}{\relax}
\BIBdecl

\bibitem{mahalanobis1936}
P.~C. Mahalanobis, ``On the generalized distance in statistics,''
  \emph{Proceedings of the National Institute of Sciences of India}, vol.~2,
  no.~1, pp. 49--55, April 1936.

\bibitem{weinberger09LMNN}
K.~Q. Weinberger and L.~K. Saul, ``Distance metric learning for large margin
  nearest neighbor classification,'' \emph{Journal of Machine Learning
  Research}, vol.~10, no.~2, pp. 207--244, Feb. 2009.

\bibitem{qi09icml}
G.-J. Qi, J.~Tang, Z.-J. Zha, T.-S. Chua, and H.-J. Zhang, ``An efficient
  sparse metric learning in high-dimensional space via {$l_1$}-penalized
  log-determinant regularization,'' in \emph{ICML}, June 2009, pp. 841--848.

\bibitem{lim13icml}
D.~Lim, G.~Lanckriet, and B.~McFee, ``Robust structural metric learning,'' in
  \emph{ICML}, June 2013, pp. 615--623.

\bibitem{liu15aaai}
W.~Liu, C.~Mu, R.~Ji, S.~Ma, J.~R. Smith, and S.-F. Chang, ``Low-rank
  similarity metric learning in high dimensions,'' in \emph{AAAI}, Jan. 2015,
  p. 2792–2799.

\bibitem{zadeh16GMML}
P.~Zadeh, R.~Hosseini, and S.~Sra, ``Geometric mean metric learning,'' in
  \emph{ICML}, June 2016, pp. 2464--2471.

\bibitem{Parikh31}
N.~Parikh and S.~Boyd, ``Proximal algorithms,'' \emph{Foundations and Trends in
  Optimization}, vol.~1, no.~3, pp. 127--239, Jan. 2014.

\bibitem{mu16aaai}
Y.~Mu, ``Fixed-rank supervised metric learning on {R}iemannian manifold,'' in
  \emph{AAAI}, Feb. 2016, pp. 1941--1947.

\bibitem{zhang17aaai}
J.~Zhang and L.~Zhang, ``Efficient stochastic optimization for low-rank
  distance metric learning,'' in \emph{AAAI}, Feb. 2017, pp. 933--939.

\bibitem{yang2018apsipa}
C.~Yang, G.~Cheung, and V.~Stankovic, ``Alternating binary classifier and graph
  learning from partial labels,'' in \emph{APSIPA}, Nov. 2018, pp. 1137--1140.

\bibitem{biyikoglu2005nodal}
T.~Biyikoglu, J.~Leydold, and P.~F. Stadler, ``Nodal domain theorems and
  bipartite subgraphs,'' \emph{The {E}lectronic {J}ournal of {L}inear
  {A}lgebra}, vol.~13, pp. 344--351, Jan. 2005.

\bibitem{bai19icassp}
Y.~Bai, G.~Cheung, F.~Wang, X.~Liu, and W.~Gao, ``Reconstruction-cognizant
  graph sampling using {G}ershgorin disc alignment,'' in \emph{ICASSP}, May
  2019, pp. 5396--5400.

\bibitem{bai19tsp}
Y.~Bai, F.~Wang, G.~Cheung, Y.~Nakatsukasa, and W.~Gao, ``Fast graph sampling
  set selection using {G}ershgorin disc alignment,'' \emph{arXiv}, 2019.

\bibitem{gahc}
R.~S. Varga, \emph{{G}ershgorin and his circles}.\hskip 1em plus 0.5em minus
  0.4em\relax Springer, 2004.

\bibitem{co1998}
C.~Papadimitriou and K.~Steiglitz, \emph{Combinatorial Optimization}.\hskip 1em
  plus 0.5em minus 0.4em\relax Dover Publications, Inc, 1998.

\bibitem{pmlr-v28-jaggi13}
M.~Jaggi, ``Revisiting {Frank-Wolfe}: Projection-free sparse convex
  optimization,'' in \emph{ICML}, Jun. 2013, pp. 427--435.

\bibitem{Knyazev01}
A.~V. Knyazev, ``Toward the optimal preconditioned eigensolver: Locally optimal
  block preconditioned conjugate gradient method,'' \emph{SIAM Journal on
  Scientific Computing}, vol.~23, no.~2, pp. 517--541, 2001.

\bibitem{shuman13spm}
D.~I. Shuman, S.~K. Narang, P.~Frossard, A.~Ortega, and P.~Vandergheynst, ``The
  emerging field of signal processing on graphs: Extending high-dimensional
  data analysis to networks and other irregular domains,'' \emph{IEEE Signal
  Processing Magazine}, vol.~30, pp. 83--98, May 2013.

\bibitem{fsp2014}
M.~Vetterli, J.~Kovacevic, and V.~Goyal, \emph{Foundations of Signal
  Processing}.\hskip 1em plus 0.5em minus 0.4em\relax Cambridge University
  Press, 2014.

\bibitem{irregraph}
M.~Milgram, ``Irreducible graphs,'' \emph{Journal Of Combinatorial Theory (B)},
  vol.~12, pp. 6--31, Feb. 1972.

\bibitem{ma2012}
R.~Horn and C.~Johnson, \emph{Matrix Analysis}.\hskip 1em plus 0.5em minus
  0.4em\relax Cambridge University Press, 2012.

\bibitem{co2009}
S.~Boyd and L.~Vandenberghe, \emph{Convex Optimization}.\hskip 1em plus 0.5em
  minus 0.4em\relax Cambridge University Press, 2009.

\bibitem{pang2017graph}
J.~Pang and G.~Cheung, ``Graph {L}aplacian regularization for image denoising:
  Analysis in the continuous domain,'' \emph{IEEE Transactions on Image
  Processing}, vol.~26, no.~4, pp. 1770--1785, April 2017.

\bibitem{bai18}
Y.~Bai, G.~Cheung, X.~Liu, and W.~Gao, ``Graph-based blind image deblurring
  from a single photograph,'' \emph{IEEE Transactions on Image Processing},
  vol.~28, no.~3, pp. 1404--1418, March 2019.

\bibitem{liu19contrast}
X.~Liu, G.~Cheung, X.~Ji, D.~Zhao, and W.~Gao, ``Graph-based joint
  dequantization and contrast enhancement of poorly lit {JPEG} images,''
  \emph{IEEE Transactions on Image Processing}, vol.~28, no.~3, pp. 1205--1219,
  March 2019.

\bibitem{liu2016random}
X.~Liu, G.~Cheung, X.~Wu, and D.~Zhao, ``Random walk graph {L}aplacian-based
  smoothness prior for soft decoding of {JPEG} images,'' \emph{IEEE
  Transactions on Image Processing}, vol.~26, no.~2, pp. 509--524, Feb. 2017.

\bibitem{Zhu:2003:SLU:3041838.3041953}
X.~Zhu, Z.~Ghahramani, and J.~Lafferty, ``Semi-supervised learning using
  {G}aussian fields and harmonic functions,'' in \emph{ICML}, Aug. 2003, pp.
  912--919.

\bibitem{maoapsipa16}
Y.~Mao, G.~Cheung, C.-W. Lin, and Y.~Ji, ``Joint learning of similarity graph
  and image classifier from partial labels,'' in \emph{APSIPA}, Dec. 2016, pp.
  1--4.

\bibitem{DBLP:journals/corr/abs-1907-09138}
W.~Hu, X.~Gao, G.~Cheung, and Z.~Guo, ``Feature graph learning for 3d point
  cloud denoising,'' \emph{arXiv}, 2019.

\end{thebibliography}

\end{document}